\documentclass{article}

\usepackage{pdfpages}
\usepackage{amsmath}
\usepackage{amsthm}
\usepackage{amsfonts}
\usepackage{amsbsy}
\usepackage{dsfont}
\usepackage{booktabs}       % professional-quality tables
\usepackage{tikz}
\usetikzlibrary{arrows}

\usepackage{times}
\usepackage{graphicx} % more modern
\usepackage{subfigure} 

\usepackage{natbib}

\usepackage{algorithm}
\usepackage{algorithmic}

\usepackage{hyperref}

\newtheorem{theorem}{Theorem}[section]

\newtheorem{proposition}[theorem]{Proposition}

\newenvironment{definition}[1][Definition]{\begin{trivlist}
\item[\hskip \labelsep {\bfseries #1}]}{\end{trivlist}}

\usepackage[accepted]{icml2016}

\icmltitlerunning{Structured Prediction by Conditional Risk Minimization}

\begin{document} 

\twocolumn[
\icmltitle{Structured Prediction by Conditional Risk Minimization}

\icmlauthor{Chong Yang Goh{\normalfont, MIT}}{cygoh@mit.edu}
\icmlauthor{Patrick Jaillet{\normalfont, MIT}}{jaillet@mit.edu}
%\icmlauthor{Chong Yang Goh}{cygoh@mit.edu}
%\icmladdress{MIT,
%            77 Massachusetts Ave, Cambridge, MA 02139}
%\icmlauthor{Patrick Jaillet}{jaillet@mit.edu}
%\icmladdress{MIT,
%            77 Massachusetts Ave, Cambridge, MA 02139}

\icmlkeywords{structured prediction, combinatorial optimization, kernel methods}

\vskip 0.3in
]

\begin{abstract} 
    We propose a general approach for supervised learning with structured output spaces, such as combinatorial and polyhedral sets, that is based on minimizing estimated conditional risk functions. Given a loss function defined over pairs of output labels, we first estimate the conditional risk function by solving a (possibly infinite) collection of regularized least squares problems. A prediction is made by solving an inference problem that minimizes the estimated conditional risk function over the output space. We show that this approach enables, in some cases, efficient training and inference without explicitly introducing a convex surrogate for the original loss function, even when it is discontinuous. Empirical evaluations on real-world and synthetic data sets demonstrate the effectiveness of our method in adapting to a variety of loss functions.
\end{abstract} 

\section{Introduction}
Many important tasks in machine learning involve predicting output labels that must satisfy certain joint constraints. These constraints help restrict the search space and incorporate domain knowledge about the task at hand. For example, in hierarchical multi-label classification, the goal is to predict, given an input, a set of labels that satisfy hierarchical constraints imposed by a known taxonomy. More generally, the output space can also depend on the observed inputs, which is commonly the case in language, vision and speech applications. Collectively, these problems are broadly referred to as \emph{structured prediction}. % In multilabel ranking, the outputs are permutations over a set of labels. 

Formally, the goal of structured prediction is to learn a predictor $h:\mathcal{X}\mapsto\mathcal{Y}$ that maps an input $x\in\mathcal{X}$ to some $y$ in a structured output space $\mathcal{Y}$. Unlike the classical problems of binary classification ($\mathcal{Y}=\{-1,1\}$) and regression ($\mathcal{Y}=\mathbb{R}$), a defining challenge in structured prediction is that $\mathcal{Y}$ can be a highly complex space, usually representing combinatorial structures like trees, matchings, and vertex label assignments on graphs, or any arbitrary set of vectors in a real vector space. The difficulty of the task naturally depends on the geometry of $\mathcal{Y}$ and the choice of loss function $\ell:\mathcal{Y}\times\mathcal{Y}\mapsto\mathbb{R}_+$, which specifies the penalty associated with a pair of predicted and realized outputs.

Most existing structured prediction methods can be understood as learning a scoring function $F(y,x)$, which assigns to each $y\in\mathcal{Y}$ its compatibility score with some input $x$. A prediction is made by solving an \emph{inference problem} that finds some $y^*\in\arg\max_{y\in\mathcal{Y}} F(y,x)$. Conditional random fields (CRFs) \cite{lafferty01}, Max-Margin Markov Networks (M$^3$N) \cite{taskar03} and Structured Support Vector Machines (SSVM) \cite{tsochantaridis05} are three commonly used methods that fit into this framework. CRFs directly model the conditional distribution $F(y,x):=p_{Y|X}(y|x)$ with a graphical model, thus maximizing the compatibility score is equivalent to maximum a posteriori estimation. In M$^3$N and SSVM, $F(y,x)$ is expressed in a linear (or log linear) form $\langle w,\phi(x,y) \rangle$, where $\phi(x,y)$ is a joint feature representation of the input-output pair. Here, the weight vector $w$ is found by solving a max-margin problem that maximizes separation of the true labels from others by their score differences.%, which are usually augmented with a loss function.

A central question that concerns the training of these structured prediction models is how to adapt them to a chosen loss function $\ell$, which reflects how the model's performance is to be measured. This is usually done by empirical risk minimization (ERM). However, minimizing the empirical risk exactly is computationally nontrivial even for simple tasks, such as predicting binary labels under the zero-one loss \cite{feldman2012}, due to discontinuities of the objective function. In M$^3$N and SSVM, this issue is addressed by substituting the original loss with a convex surrogate that is more amenable to computation. In CRFs, several works have also sought to incorporate loss functions into the training of probabilistic models \cite{gross06,volkovs11}; likewise, they require relaxing the original loss to be tractable. One known issue of minimizing a convex surrogate is that it can result in suboptimal performance in terms of the actual loss, due to looseness of the upper bound \cite{chapelle09}. Nevertheless, if we work backwards by considering what constitutes an ideal learning outcome, the notion of Bayes optimality succinctly characterizes the goal: we seek a predictor $h^*$ that maps each input $x$ to an output $y^*$ minimizing the \emph{conditional risk function} $R(y|x):=\mathbb{E}_{Y|X}[\ell(y,Y)|X=x]$. Motivated by this perspective, we explore two questions in this paper:
\begin{enumerate}
\item \emph{Can a structured prediction framework be developed based on discriminative modeling of the conditional risk function $R(\cdot|\cdot)$ and its subsequent minimization?}
\item \emph{What are the computational and statistical properties of such algorithms?}
\end{enumerate}

To answer the first question, we propose a framework for structured prediction that is based on \emph{estimated conditional risk minimization} (ECRM), which can also be characterized as one that learns a scoring function $F(y,x)$. Specifically, the scoring function (to be minimized over) can be interpreted as a direct estimate of the conditional risk function, $F(y,x)\approx R(y|x)$. We derive a closed form expression for $F$ under a rich nonparametric modeling of the conditional risk function, obtained by solving a collection of regularized least squares problems (Section 2).

To answer the second question, we apply the method to a class of structured prediction tasks with combinatorial outputs and additive loss functions, and derive sufficient conditions under which the inference problem can be solved exactly and efficiently (Section 3). We complement these computational results with a statistical analysis of the algorithm, providing a generalization bound that holds under fairly general settings (Section 4).

Finally, we discuss some extensions with additive models and joint feature maps (Section 5). We then evaluate the ECRM approach on two tasks, hierarchical multilabel classification (discrete outputs) and prediction of flows in a network (continuous outputs), with real-world and synthetic data sets. The experimental results demonstrate how the algorithm adapts to a variety of loss functions and compare favorably with existing methods (Section 6).
\section{Estimated Conditional Risk Minimization}
\subsection{Preliminaries}
Throughout this paper, we represent $\mathcal{X}\subseteq\mathbb{R}^p$ and $\mathcal{Y}\subseteq\mathbb{R}^d$ as sets of vectors in a real vector space. Depending on the application, $\mathcal{Y}$ can be discrete (e.g., a combinatorial subset of $\{0,1\}^d$) or otherwise (e.g. a polytope). As in most supervised learning settings, we have a training set, $S=[(x^{(i)},y^{(i)})\in\mathcal{X}\times\mathcal{Y}:i=1,\ldots,m]$, where the samples are i.i.d. and drawn from some fixed distribution $\mathbb{P}_{X,Y}$. We define $R(y|x):=\mathbb{E}_{Y|X}[\ell(y,Y)|X=x]$ to be the conditional risk of predicting label $y$ having observed an input $x$, and use $\hat{R}$ to denote its estimate. Here the expectation is defined with respect to the conditional distribution $\mathbb{P}_{Y|X}$.

In summary, the proposed method can be described in two steps:
\begin{enumerate}
\item \textbf{Training}: Learn a set of functions indexed by $y$, $\{\hat{R}(y|\cdot):y\in\mathcal{Y}\}$ by solving a (possibly infinite) collection of regularized least squares problems.
\item \textbf{Prediction}: Given an input $x$, predict $h(x)=y^*$ by solving an inference problem,
\begin{equation}
y*\in\arg\min_{y\in\mathcal{Y}}\hat{R}(y|x).
\end{equation}
\end{enumerate}
\subsection{Training}
We begin with two basic observations that hold under weak regularity assumptions:\footnote{See, e.g., \cite{lehmann98}.} for any fixed $y\in\mathcal{Y}$,
\begin{enumerate}
\item The random loss function $\ell(y,Y)$ can be written as $\ell(y,Y)=R(y|X)+\varepsilon(y,X)$, where $\varepsilon(y,X)$ is a zero-mean random variable that depends (randomly) on $X$ and (deterministically) on $y$.
\item $R(y|\cdot)\in\arg\min_{f:\mathcal{X}\mapsto\mathbb{R}}\mathbb{E}[(f(X)-\ell(y,Y))^2]$.
\end{enumerate}
This decomposition suggests a regression approach to estimating $R(y|\cdot)$: for every $y\in\mathcal{Y}$, we posit that $R(y|\cdot)$ lies in some function space $\mathcal{H}\subseteq\{f:\mathcal{X}\mapsto\mathbb{R}\}$ and estimate it by solving a regularized least squares problem,\footnote{To simplify exposition, we omit modeling with intercept here and discuss how it can be accounted for in the Appendix.}
\begin{equation} \label{eq:rls}
\hat{R}(y|\cdot)\in\arg\min_{f\in\mathcal{H}}\dfrac{1}{m}\sum_{i=1}^m (f(x^{(i)})-\ell(y,y^{(i)}))^2 + \lambda\|f\|^2_{\mathcal{H}},
\end{equation}
where $\|\cdot\|_\mathcal{H}$ is a norm over $\mathcal{H}$ and $\lambda$ is a regularization parameter. To model a rich nonparametric class of conditional risk functions, we define $\mathcal{H}$ as a reproducing kernel Hilbert space (RKHS) spanned by a real, symmetric positive definite kernel $k:\mathcal{X}\times\mathcal{X}\mapsto\mathbb{R}$. This reduces the problem to kernel ridge regression (KRR) \cite{saunders98}, which admits the following closed form solution.
\begin{proposition}\label{pro:rls}
For every $y\in\mathcal{Y}$, an optimal solution for the regularized least squares problem can be expressed as
\begin{equation}\label{eq:rlssol}
\hat{R}(y|\cdot)=\sum_{i=1}^m w_i(\cdot)\ell(y,y^{(i)}),
\end{equation}
where $w(x):=(K+m\lambda I)^{-1}v(x)$ is a $m\times 1$ weight vector, $K:=[k(x^{(i)},x^{(j)})]_{ij},\forall i,j\in\{1,\ldots,m\}$ is a $m\times m$ gram matrix and $v(x):=[k(x,x^{(i)})]_{i=1}^m$ is a $m\times 1$ vector.
\end{proposition}
\begin{proof}
For a fixed $y$, the \emph{representer's theorem} \cite{scholkopf01} implies that any optimal solution of the problem (\ref{eq:rls}) lies in the span of $\{k(\cdot,x^{(i)})\}_{i=1}^m$, i.e., there exists some $\alpha\in\mathbb{R}^m$ such that $\hat{R}(y|\cdot)=\sum_{i=1}^m \alpha_i k(\cdot,x^{(i)})$. This also implies that $\|\hat{R}(y|\cdot)\|^2_\mathcal{H}=\alpha^T K \alpha$, where $K$ is defined above. Thus it suffices to consider an equivalent optimization problem over $\alpha$,
\begin{equation*}\label{eq:rlsalpha_rss}
\min_{\alpha\in\mathbb{R}^m}\dfrac{1}{m}\sum_{i=1}^m \left(\sum_{j=1}^m \alpha_j k(x^{(i)},x^{(j)}) -\ell(y,y^{(i)})\right)^2 + \lambda\alpha^T K \alpha.
\end{equation*}
Denoting $L_y:=[\ell(y,y^{(1)})\;\cdots\;\ell(y,y^{(m)})]$ as the vector of observed losses, this can be written compactly as
\begin{equation}\label{eq:rlsalpha_matrix}
\min_{\alpha\in\mathbb{R}^m} \dfrac{1}{m}\| K\alpha - L_y\|_2^2 + \lambda\alpha^T K \alpha.
\end{equation}
Since the problem is convex (which follows from $k$ being a symmetric, positive definite kernel), the first order optimality conditions are necessary and sufficient. Taking the derivative of the objective function and equating it to zero, we have
\begin{gather*}
2K^T K\alpha^*_y - 2KL_y + 2m\lambda K\alpha^*_y = 0 \\
\alpha^*_y = (K+m\lambda I)^{-1} L_y.
\end{gather*}
For any $x\in\mathcal{X}$, we have $\hat{R}(y|x)=\sum_{i=1}^m \alpha^*_{y,i} k(x,x^{(i)})=\alpha^{*T}_y v(x)$. Substituting $\alpha^*_y$ with the above solution, we obtain $L_y^T(K+m\lambda I)^{-1} v(x)=\sum_{i=1}^m w_i(x) \ell(y,y^{(i)})$, thus proving our claim.
\end{proof}
Because the above holds true for all $y\in\mathcal{Y}$, we have a complete characterization of $\hat{R}(\cdot|x)$ as the weighted sum of the individual loss functions $\{\ell(\cdot,y^{(i)})\}_{i=1}^m$ induced by $y^{(1)},\ldots,y^{(m)}$, with weights $w(x)$ that only depend on $x$ and $x^{(1)},\ldots,x^{(m)}$.

So far, our derivation of $\hat{R}(\cdot|\cdot)$ has not relied on any particular assumption about the output space $\mathcal{Y}$ and the loss function $\ell$, which makes it broadly applicable. But certain problems admit structures that presumably can be exploited for better generalization performance: for example, the loss function may be additive over substructures of $\mathcal{Y}$, in which case modeling the conditional risk function additively can be useful for incorporating additional domain knowledge about the task. We defer the discussion of these extensions to Section \ref{sec:extension}.
\subsection{Prediction}
Given an input $x$, we first compute the weight vector $w(x)=(K+m\lambda I)^{-1}v(x)$. This is done by first forming $v(x)\in\mathbb{R}^m$ and then either solving a linear system, or multiplying $v(x)$ with the inverted matrix if already pre-computed in training.\footnote{In large-scale problems, we can use a low-rank approximation of $K$ to reduce storage and computational requirements. See, e.g., \cite{kumar09,si14}.}

A prediction $y^*=\hat{h}(x)$ is computed by solving an auxiliary optimization problem that minimizes the estimated conditional risk $\hat{R}(y|x)$,
\begin{equation}
y^*\in\arg\min_{y\in\mathcal{Y}}\sum_{i=1}^m w_i(x) \ell(y,y^{(i)}).\label{eq:auxprob}
\end{equation}
The difficulty of this problem crucially depends on the geometry of $\mathcal{Y}$, the choice of $\ell$ and in some cases the signs of $w(x)$. In Section 3, we will characterize a class of problems for which this inference problem can solved efficiently. In some cases, this can be done even if $\ell$ is discontinuous, without requiring the use of surrogate loss.
\subsection{Related Work}
\subsubsection{Regularized Least Squares Classification}
We show that applying ECRM in binary classification with zero-one loss is equivalent to Regularized Least Squares Classification (RLSC)\footnote{RLSC is closely related to Least Squares SVM \cite{suykens99}, which differs only in that it includes an unpenalized intercept.} \cite{rifkin03}. Thus ECRM can be viewed as a generalization of RLSC to structured output spaces. With $\mathcal{Y}=\{-1,1\}$ and $\ell(y,y')=\mathds{1}(y\neq y')$, the conditional risk $R(y|x)$ is simply the conditional probability of misclassification, $\mathbb{P}(y\neq Y|X=x)$. Applying (\ref{eq:auxprob}), we obtain the following ECRM classification rule,
\begin{equation}
\hat{h}(x)=1\text{ iff }\sum_{i=1}^m w_i(x)y^{(i)} \geq 0. \label{eq:ecrmbinary}
\end{equation}
In RLSC, we characterize the hypothesis space in the form of $h(x)=\text{sgn}(f(x))$ for some function $f:\mathcal{X}\mapsto\mathbb{R}$, and estimate an $f^*$ by solving a regularized least squares problem:
$f^*(\cdot)\in\arg\min_{f\in\mathcal{H}}\frac{1}{m}\sum_{i=1}^m (f(x^{(i)})-y^{(i)})^2 + \lambda\|f\|^2_{\mathcal{H}}.$ Here $\mathcal{H}$ is a RKHS associated with a kernel $k$. By the representer's theorem, the solution can be expressed in closed form as $f^*(x)=\sum_{i=1}^m w_i(x)y^{(i)}$, where $w(x)$ is defined as in Proposition \ref{pro:rls}. The resulting classification rule $h(x)=\text{sgn}(f^*(x))$ is exactly the same as (\ref{eq:ecrmbinary}), thus proving the equivalence.

\subsubsection{KRR-based Methods}
Kernel Dependency Estimation (KDE) \cite{weston03} and its extensions \cite{cortes05} are also structured prediction methods based on KRR. However, KDE differs from our approach in that regression is used to learn a direct mapping from inputs to a feature space associated with the output space $\mathcal{Y}$, whose output is then mapped from the feature space back to $\mathcal{Y}$ by solving a \emph{pre-image} problem. Adapting this method to a loss function requires the output feature space to be represented by a properly crafted kernel. Except in a few special cases, this kernel representation results in a pre-image problem that is hard to solve \cite{giguere15}.

Recently, Ciliberto et. al. \yrcite{ciliberto16} formulated a generalization of KDE by showing that a broad class of loss functions naturally induce an embedding of structured outputs in a feature space, in which the loss can be expressed as an inner product form. They showed that the training problem also reduces to KRR, but the inference problem can be solved without explicitly computing an inverse mapping (thus avoiding the pre-image problem of KDE). Our method can be seen as a novel, conceptually simpler derivation of some of these results through the lens of conditional risk minimization. In addition, our method can be generalized to additive models with joint kernels (Section \ref{sec:extension}).

\section{Computational Properties}
\label{sec:computational}
To characterize problems for which efficient inference in (\ref{eq:auxprob}) is possible, we focus on $\mathcal{Y}\subseteq\{0,1\}^d$ and any loss function $\ell(y,y')$ that can be expressed in an additive form, $\ell(y,y'):=\sum_{j=1}^d \ell_j(y_j,y')$, where $\ell_j:\{0,1\}\times\mathcal{Y}\mapsto\mathbb{R}_+$. We define $\mathcal{Y}$ to be a set of points in $\{0,1\}^d$ that satisfy the following linear constraints,
\begin{equation} \label{eq:linconstraints}
  \begin{aligned}
    & & a_{i_1}^T y \leq b_{i_1},\forall i_1\in I_1, \\
    & & a_{i_2}^T y \geq b_{i_2},\forall i_2\in I_2, \\
    & & a_{i_3}^T y = b_{i_3},\forall i_3\in I_3, \\
  \end{aligned}
\end{equation}
where $I_1,I_2,I_3$ are (possibly empty) disjoint sets of indices such that their union is $\{1,\ldots,n\}$, and $a_i\in\mathbb{R}^d,b_i\in\mathbb{R},\forall i$. In other words, there are a total of $n$ linear constraints, each can either be an inequality or an equality. This characterization of $\mathcal{Y}$ is fairly general as many objects of interest, including matchings, permutations and label assignments in graphs, can be represented as such. For convenience, we will define $A:=[a_1 \ldots a_n]^T$ to be the \emph{constraint matrix}, $b:=[b_1 \ldots b_n]^T$, and $\mathcal{Z}\subseteq\mathbb{R}^d$ as the polyhedron characterized by these linear constraints. Under this loss and our definition of $\mathcal{Y}$, the inference problem in (\ref{eq:auxprob}) is a discrete optimization problem,
\begin{equation}\label{eq:inference_ip}
  \begin{aligned}
    & \text{minimize }
    & & \sum_{i=1}^m\left(\sum_{j=1}^d \ell_{j}(y_j,y^{(i)})\right) w_i(x) \\
    & \text{subject to }
    & & y \in \{0,1\}^d\cap \mathcal{Z}
  \end{aligned}
\end{equation}

This problem is $\mathcal{NP}$-complete in general. However, we will show that an interesting subclass can be solved \emph{exactly} and efficiently by linear programming relaxation, obtained by relaxing constraints $y\in\{0,1\}^d$ to $0\leq y\leq 1$, and rewriting the objective as a linear function of $y$. One property of matrices that is useful for this purpose is \emph{total unimodularity}, defined as follows.
\begin{definition} (Total Unimodularity) A matrix $A$ is totally unimodular if every square submatrix of $A$ has determinant $0,-1$ or $1$.
\end{definition}

The connection between total unimodularity and exactness of linear programming relaxation is well known \cite{schrijver98}. Applying it to (\ref{eq:inference_ip}), we obtain the following result.

\begin{theorem}\label{thm:lprelax}
If $A$ is totally unimodular and $b\in\mathbb{Z}^n$, then for any $w(x)\in\mathbb{R}^m$, an optimal solution of the inference problem can be found by solving a linear program,
\begin{equation*}
  \begin{aligned}
    & \textnormal{minimize}
    & & \sum_{j=1}^d\left(\sum_{i=1}^m (\ell_j(1,y^{(i)})-\ell_j(0,y^{(i)})) w_i(x)\right) y_j \\
    & \textnormal{subject to }
    & & 0 \leq y_j \leq 1, \forall j=1,\ldots,d\\
    &&& y\in \mathcal{Z} \\
  \end{aligned}
\end{equation*}
\end{theorem}
While total unimodularity is not a necessary condition for the above relaxation to be exact,\footnote{To see why, we can add redundant constraints to $A$ to guarantee that it is not totally unimodular, without altering the polyhedron.} it is a property that can be tested in polynomial time given a constraint matrix $A$ \cite{truemper90}. More generally, there are many results on classes of matrix that satisfy this property \cite{conforti14}. In what follows, we will demonstrate how Theorem \ref{thm:lprelax} can be specialized to two examples.
\subsubsection{Example 1: Hierarchical Multilabel Classification (HMC)} \label{ssec:hmc}
In HMC, we have a set of labels $\mathcal{V}=\{1,\ldots,d\}$ organized in a hierarchy (e.g., of topics in text classification). Our goal is to predict a subset of $\mathcal{V}$ that corresponds to an input $x$. Let $y_j\in\{0,1\},j=1,\ldots,d$ denote whether each label $y_j$ is chosen (1) or not (0). In addition to choosing a subset of $\mathcal{V}$, we require that $y$ satisfies the following hierarchical constraints:
\begin{itemize}
\item For each $j\in\mathcal{V}$: if $y_j=1$, then $y_k=1,\forall k\in\mathcal{P}(j)$. Here $\mathcal{P}(j)$ denotes the set of immediate parent labels under which $j$ belongs in the hierarchy.
\end{itemize}
The hierarchy is commonly represented as a tree or (more generally) a directed acyclic graph (DAG), with each node being a label in $\mathcal{V}$ and each arc $(k,j)$ encoding a parent-child relation, $k\in\mathcal{P}(j)$. Formally, we define the DAG as $G=(\mathcal{V},\mathcal{A})$, where $\mathcal{V}$ is defined as above and $(k,j)\in\mathcal{A}$ iff $k\in\mathcal{P}(j)$. The output space $\mathcal{Y}$ can be succintly described with $|\mathcal{A}|$ linear constraints (in addition to $y\in\{0,1\}^d$),
\begin{equation*}
y_j \leq y_k, \forall (k,j)\in\mathcal{A}. \\
\end{equation*}
We can express these constraints as $Ay\leq 0$, where $A\in\mathbb{R}^{|\mathcal{A}|\times|\mathcal{V}|}$ is the \emph{hierarchical constraint matrix} satisfying\footnote{For an arc $a=(k,j)$, $k$ is the \emph{head} and $j$ is the \emph{tail}.}
\begin{equation*}
A_{aj}=
\begin{cases}
  -1, & \text{if } j \text{ is the head of arc $a$}, \\
  1, & \text{if } j \text{ is the tail of arc $a$}, \\
  0, & \text{otherwise.}
\end{cases}
\end{equation*}

\begin{proposition}\label{pro:hmc}
For any directed graph $G=(\mathcal{V},\mathcal{A})$, let $A$ be its corresponding hierarchical constraint matrix. Then $A$ is totally unimodular.
\end{proposition}

Proposition \ref{pro:hmc} implies that the constraint conditions in Theorem \ref{thm:lprelax} are satisfied (with $b=0$), allowing us to apply it to any loss function in the form $\sum_{j=1}^d \ell_j(y_j,y')$. We now show that two commonly used loss functions for HMC, the Hamming loss and the Hierarchical loss \cite{cesa06}, can be expressed as such.

The Hamming loss, $\ell_{\text{hm}}(y,y')=\sum_{j=1}^d \mathds{1}(y_j\neq y'_j)$, which penalizes label-wise errors, clearly satisfies this property with $\ell_j(y_j,y'):=\mathds{1}(y_j\neq y'_j)$. So by Theorem \ref{thm:lprelax}, the inference problem can solved by linear programming with $\ell_j(1,y^{(i)})-\ell_j(0,y^{(i)})=1-2y_j^{(i)}$. The Hierarchical loss differs from the Hamming loss in that it penalizes an incorrect label at a node only if all its ancestor nodes are correctly labeled,
\begin{equation*}
\ell_{\text{hr}}(y,y')=\sum_{j=1}^d c_j \mathds{1}(y_j\neq y'_j,y_k=y'_k,\forall k\in\mathcal{Q}(j))
\end{equation*}
Here $c_j\in\mathbb{R}$ is a penalization factor and $\mathcal{Q}(j)$ denotes the (possibly empty) set of ancestors of node $j$ in the hierarchy. Intuitively, if $G$ is an \emph{arborescence}, i.e., a rooted directed tree with all arcs pointing away from the root, then $\ell_{\text{hr}}$ penalizes the mistakes along every path from the root at most once. To weigh mistakes closer to the root more heavily, $c_j$ is usually set to $1$ for the root node. For all other nodes, we let $c_j:=c_{p(j)}/|\mathcal{S}(j)|$, where $p(j)$ is the parent of node $j$ and $\mathcal{S}(j)$ is the set of its siblings (including node $j$). We will refer to this variant of $\ell_{\text{hr}}$ as the \emph{sibling-weighted} Hierarchical loss. Despite introducing complex dependencies between labels, the proposition below shows that $\ell_{\text{hr}}$ also admits a linear additive form that allows us to apply Theorem \ref{thm:lprelax} directly.

\begin{proposition}\label{pro:hmc}
For any arborescence $G=(\mathcal{V},\mathcal{A})$ with root $s\in\mathcal{V}$ and any pair $y,y'\in\mathcal{Y}$, the Hierarchical loss $\ell_{\text{hr}}(y,y')$ with respect to $G$ is equivalent to
\begin{equation*}
c_s(y_s+y_s'-2y_s'y_s)+\sum_{(j,k)\in\mathcal{A}}c_k(y_k'y_j + (y_j'-y_j'y_k'-y_k')y_k).
\end{equation*}
\end{proposition}

Under both $\ell_{\text{hm}}$ and $\ell_{\text{hr}}$, the size of the resulting linear program (in terms of variable and constraint counts) scales linearly with $d$ and $|\mathcal{A}|$. The cost coefficients of the linear program can be computed in $O(md)$ given $w(x)$.

Finally, we note that existing algorithms for HMC are either specialized to tree hierarchies \cite{rousu06}, require stronger conditions for exact inference in DAG \cite{bi11}, or is not directly applicable to both loss functions \cite{deng14}. To the best of our knowledge, ours is the first general formulation of HMC with provably exact inference under both loss functions.
\subsubsection{Example 2: Multilabel Ranking}
Suppose we are interested in predicting a ranking over all labels, rather than only choosing a subset. We consider a setting where the training set consists of complete permutations $\sigma^{(1)},\ldots,\sigma^{(m)}$ over the label set $\mathcal{V}$ and their associated inputs. The goal is to learn to predict a permutation $\sigma$ given an input $x$, where $\sigma(j)$ indicates the rank of label $j$ for every $j=1,\ldots,d$. To measure the loss, we will use the \emph{Spearman's footrule distance}, $\tilde{\ell}(\sigma,\sigma')=\sum_{j=1}^d |\sigma(j)-\sigma'(j)|$, which sums the absolute differences of the two ranks over all labels. By representing $\sigma$ as a binary vector, we will show that the inference problem can be solved by exact linear programming relaxation.

Let us define $y_{j,k}\in\{0,1\},\forall j,k\in\{1,\ldots,d\}$, as a vector in $\mathbb{R}^{d^2}$ that corresponds to some $\sigma$ as follows: ($\forall j,k$) $y_{jk}=1$ iff $\sigma(j)=k$. The set of such vectors (each corresponding to a distinct $\sigma$), denoted $\mathcal{Y}\subseteq\{0,1\}^{d^2}$, are exactly characterized by the following linear constraints.
\begin{equation}
  \begin{aligned}
    & \textstyle \sum_{k=1}^d y_{jk} = 1,\forall j=1,\ldots,d \\
    & \textstyle \sum_{j=1}^d y_{jk} = 1,\forall k=1,\ldots,d
  \end{aligned}
\end{equation}
This can be interpreted as the set of perfect matchings in a complete bipartite graph. With this representation, we can equivalently write $\tilde{\ell}$ as $\ell(y,y')=\sum_{j,k,l} |k-l|y'_{jl}y_{jk}$, where each summation over $j,k,l$ is from $1$ to $d$. This is again a special case of the additive loss function defined earlier. Together with the fact that the constraint matrix associated with bipartite matchings is known to be totally unimodular \cite{schrijver98}, we can apply Theorem \ref{thm:lprelax} to reduce the inference problem to a min-cost assignment problem.

\section{Statistical Properties}
We derive a generalization bound for ECRM based on \emph{algorithmic stability} \cite{bousquet02,mukherjee02} that holds under general assumptions. Previous theoretical analysis of structured prediction methods typically bounds the empirical risk in terms of \emph{margin loss} for linear discriminative models \cite{london13,cortes16}. In contrast, our result holds under a nonparametric setting, applies to a broad class of loss function $\ell$ and output space $\mathcal{Y}$ (discrete or otherwise), and is based on a provably tighter family of surrogate losses defined parametrically with respect to some $\rho>0$,
\begin{equation}
L^\rho_{\hat{R}}(x,y):= \Phi\left(\max_{y'\in\mathcal{Y}}\left\{\ell(y',y)+\frac{1}{\rho}\Delta_{\hat{R}}(y',x)\right\}\right).
\label{eq:surrogate_loss}
\end{equation}
Here $\Delta_{\hat{R}}(y',x):=\min_{y''\in\mathcal{Y}}\hat{R}(y''|x)-\hat{R}(y'|x)$, where $\hat{R}(\cdot|\cdot)$ is an estimated conditional risk function defined in Proposition \ref{pro:rls} and $\Phi(a):=\min\{a,L\}$ for some $L>0$. Assuming that the loss function $\ell$ is bounded, we set $L:=\sup_{y,y'\in\mathcal{Y}}\ell(y,y')$ so that $L^\rho_{\hat{R}}$ never exceeds the upper bound. Intuitively, $L^\rho_{\hat{R}}$ can be understood as a Langrangian relaxation of the following optimization problem (with $\rho$ corresponding to a multiplier), the optimal cost of which is equal to the original loss $\ell(\hat{h}(x),y)$ (here, $\hat{h}(x)\in\arg\min_{y''\in\mathcal{Y}}\hat{R}(y''|x)$.\footnote{If the set of minimizers is not unique, we define $\hat{h}(x)$ to be one that incurs the highest loss $\ell(\hat{h}(x),y)$.}),
\begin{equation}
\label{eq:langrangian}
\max_{y'\in\mathcal{Y}} \ell(y',y)\text{ s.t. }\hat{R}(y'|x)\leq \min_{y''\in\mathcal{Y}}\hat{R}(y''|x).
\end{equation}
In some cases, the Lagrangian can be made tight with a sufficiently small $\rho$ (thus satisfying strong duality), as the following proposition shows.
\begin{proposition}
The function $L^\rho_{\hat{R}}$ satisfies the following properties for any given pair $(x,y)$.
\begin{enumerate}
\item Surrogacy: $L^\rho_{\hat{R}}(x,y)\geq \ell(\hat{h}(x),y)$ for any $\rho>0$.
\item Monotonicity: $L^\rho_{\hat{R}}(x,y)$ is nondecreasing in $\rho$.
\item Tightness: If $\mathcal{Y}$ is finite, then there exists some $\rho^*>0$ such that $L^\rho_{\hat{R}}(x,y)=\ell(\hat{h}(x),y),\forall \rho\in(0,\rho^*]$.
\end{enumerate}
\end{proposition}
Alternatively, we can interpret $L^\rho_{\hat{R}}$ as a variant of the \emph{structured ramp loss} \cite{chapelle09}, which has been shown to be a tighter surrogate than the margin loss. In practice, for a given pair $(x,y)$, we can compute $L^\rho_{\hat{R}}$ efficiently by linear programming if the conditions in Section \ref{sec:computational} are satisfied, because the objective function in (\ref{eq:surrogate_loss}) reduces to a weighted sum over individual loss functions induced by $y,y^{(1)}\ldots,y^{(m)}$.

Let us denote $\mathfrak{R}(\hat{h}):=\mathbb{E}_{X,Y}[\ell(\hat{h}(X),Y)]$ as the expected risk of predictor $\hat{h}$ and $\hat{\mathfrak{R}}^\rho(\hat{h}):=\frac{1}{m}\sum_{i=1}^m L^\rho_{\hat{R}}(x^{(i)},y^{(i)})$ as the empirical risk based on surrogate $L^\rho_{\hat{R}}$. We now state the main result for generalization bound below.
\begin{theorem}
Let $\hat{h}$ be an ECRM predictor trained with some kernel $k$ and regularization parameter $\lambda$. Suppose that $\sup_{y,y'\in\mathcal{Y}}\ell(y,y')\leq L$ and $\sup_{x\in\mathcal{X}}k(x,x)\leq \kappa$. Then for any $\rho>0,\lambda>0$ and $\delta\in(0,1)$, the following bound holds with probability at least $1-\delta$,
\begin{equation}
\mathfrak{R}(\hat{h}) \leq \hat{\mathfrak{R}}^\rho(\hat{h}) + \dfrac{4L\nu}{\rho m} + L\left(\dfrac{8\nu}{\rho}+1\right)\sqrt{\dfrac{\ln(1/\delta)}{2m}},
\end{equation}
where $\nu:=\kappa/\lambda + (\kappa/\lambda)^{3/2}$.
\label{thm:gen_bound}
\end{theorem}
Here, choosing $\rho$ is a matter of tradeoff between tightness of the empirical risk estimate and the error terms. In practice, we may want to optimize $\rho$ after seeing the data. To that end, Theorem \ref{thm:gen_bound} can be extended to hold uniformly over a range $\rho\in(0,B]$ at the expense of a $O(\sqrt{(\ln\ln(B/\rho))/m})$ term, by using existing techniques (see, e.g., \cite{bousquet02}).
\section{Additive Models of Conditional Risk}
\label{sec:extension}
As seen in previous examples, many loss functions considered in structured prediction are additive over substructures of $\mathcal{Y}$. It can be useful to decompose the learning problem over these substructures, so that any local features can be exploited for better generalization performance. Here we briefly discuss how our method can be extended to an additive model of conditional risk. Consider $\ell(y,y')=\sum_j \ell_j(y_{S_j},y'_{S_j})$, where each $\ell_j$ is defined over a subset of elements in $y,y'$ that correspond to the index set $S_j\subseteq\{1,\ldots,d\}$. By linearity of expectation, we can also express the conditional risk function in additive form, $R(y|x)=\sum_j R_j(y_{S_j}|x)$ with $R_j(y_{S_j}|x):=\mathbb{E}[\ell_j(y_{S_j},Y_{S_j})|X=x]$, and then estimate it by solving a multitask least squares problem,
\begin{equation} \label{eq:rls_jointkernel}
\min_{f\in\mathcal{H}}\sum_j\sum_{y_{S_j}\in\mathcal{Y}_j}\sum_{i=1}^m (f(y_{S_j},x^{(i)})-\ell_j(y_{S_j},y_{S_j}^{(i)}))^2 + \lambda\|f\|^2_{\mathcal{H}}.
\end{equation}
Here we denote $\mathcal{Y}_j$ as the set of possible values that $y_{S_j}$ can take, and $\mathcal{H}$ as a RKHS associated with a \emph{joint kernel} $K((x,y_{S{j}}),(x',y'_{S_k}))$, which defines a similarity measure between input-output pairs. Intuitively, we can view an optimal solution $f^*$ of the problem above as direct estimates of individual components of the conditional risk function, such that $f^*(y_S{_j},x)\approx R_j(y_{S_j}|x),\forall j$, and then construct $\hat{R}(y|x):=\sum_{j}f^*(y_{S_j},x)$. One advantage of this formulation over (\ref{eq:rls}) is that we have added flexibility to model conditional risk correlations between substructures in $\mathcal{Y}$ through joint kernels. For example, in HMC with Hamming loss, we can naturally decompose the problem over individual labels by having $S_j:=\{j\}$ and $\ell_j(y_{S_j},y'_{S_j}):=\mathds{1}(y_j\neq y'_j)$. To model pairwise correlations in the hierarchy, we can define $K((x,y_{S{j}}),(x',y'_{S_k})):=k(x,x')\mathds{1}(y_j=y_k)\mathds{1}(j\in\delta(k))$, where $\delta(k)$ denotes the set of adjacent nodes of $k$ (including itself). In this case, while the solution $f^*$ does not admit the form in Proposition \ref{pro:rls}, we can still express it compactly as\footnote{By representer's theorem, as in the proof of Proposition \ref{pro:rls}.}
\begin{equation*}
f^*(y_{S_j},x)=\sum_{i=1}^m \sum_{k\in\delta(j)}(\alpha_{ik1}y_j + \alpha_{ik0}(1-y_j))k(x,x^{(i)}),
\end{equation*}
for some $\alpha\in\mathbb{R}^{2md}$. Substituting this expression into (\ref{eq:rls_jointkernel}), we can find an optimal $\alpha^*$ in closed form or by convex optimization. Since $f^*(y_{S_j},x)$ is linear in $y_j$, we obtain a $\hat{R}(y|x)$ that is also linear in $y$. Therefore, we can apply Theorem \ref{thm:lprelax} to solve the inference problem by exact linear programming relaxation.

\section{Experiments}
\label{sec:experiments}
\subsection{Hierarchical Multilabel Classification}
We evaluate ECRM in the HMC task using the formulation in Section \ref{ssec:hmc}, and compare it with two existing methods, each based on a different paradigm: (i) Hierarchical Max-Margin Markov Networks (HM$^3$N): A max-margin approach specialized to tree hierarchies \cite{rousu06}. We use the source code by the original authors in the implementation. (ii) BR-SVM: A \emph{binary relevance} approach: A SVM classifier is trained for each node, and the predictions are imputed from the bottom up to satisfy the hierarchical constraints. We use LibSVM \cite{chang11} to train these classifiers.

Each method is evaluated on HMC benchmark data sets from three domains: text (\textsc{enron}, \textsc{reuters}, \textsc{wipo}), image (\textsc{imclef07a}, \textsc{imclef07d}) and functional genomics (\textsc{pheno\_go}, \textsc{spo\_go}, \textsc{pheno\_fun}, \textsc{spo\_fun}). The hierarchies associated with all data sets are trees, except for \textsc{pheno\_go} and \textsc{spo\_go}, which are loopy DAGs representing gene ontology networks. A summary of these data sets are available in the Appendix. Both ECRM and HM$^3$N are trained based on Hamming loss and sibling-weighted Hierarchical loss, respectively. BR-SVM is trained without adaptation to either losses. For consistency, we do not apply any feature selection. All three methods use RBF kernels on the image data sets, and linear kernels on the rest. The parameters are tuned on the training sets by grid search with cross validation. All tests are run on a machine with a quad-core 2.6GHz CPU with 16GB RAM.

Table \ref{tbl:hmc_results} summarizes the benchmark results under Hamming loss ($\ell_{\text{hm}}$) and Hierarchical loss ($\ell_{\text{hr}}$). Overall, ECRM outperforms the two other methods under both losses, especially on the image and genomics data sets. As shown in Figure \ref{fig:running_time}, ECRM also has the unique property that its training time does not depend on the number of labels in the hierarchy, making it scalable to large graphs. Whereas BR-SVM requires training separate classifiers and HM$^3$N requires solving a max-margin problem with constraints that scale with the hierarchy, training ECRM only involves forming the kernel matrix and computing its inversion, which is independent of the hierarchy size. In terms of inference, the average time per instance taken by ECRM, HM$^3$N and BR-SVM on the {\sc reuters} data set are $7.0$ms, $4.5$ms and $10.4$ms, respectively.

\begin{table}[t]
  \caption{The average Hamming ($\ell_\text{hm}$) and Hierarchical ($\ell_\text{hr}$) losses of HMC methods on various data sets. For each data set and loss function, the best result is typeset in bold. Some results for HM$^3$N are not available because it is not applicable to loopy graphs.}
  \label{tbl:hmc_results}
\begin{center}
\begin{small}
  \begin{tabular}{@{}lcccccc@{}}
  \toprule
   & \multicolumn{2}{c}{\textsc{ECRM}} & \multicolumn{2}{c}{\textsc{HM$^3$N}} & \multicolumn{2}{c}{\textsc{BR-SVM}} \\
  \midrule
  \textsc{data set}    & $\ell_\text{hm}$  & $\ell_\text{hr}$  & $\ell_\text{hm}$  & $\ell_\text{hr}$  & $\ell_\text{hm}$  & $\ell_\text{hr}$ \\ \midrule
  \textsc{enron}       & {\bf 3.079}       & {\bf 0.194}       & 3.785             & 0.212             & 3.400             & 0.196 \\
  \textsc{reuters}     & 1.507             & 0.079             & 1.480             & 0.082             & {\bf 1.386}       & {\bf 0.075} \\
  \textsc{wipo}        & 2.011             & 0.049             & {\bf 1.659}       & {\bf 0.048}       & 1.687             & 0.051 \\ \midrule
  \textsc{imclef07a}   & {\bf 2.689}       & {\bf 0.119}       & 3.053             & 0.134             & 2.933             & 0.132 \\
  \textsc{imclef07d}   & {\bf 3.246}       & {\bf 0.246}       & 3.413             & 0.247             & 3.250             & 0.253 \\ \midrule
  \textsc{pheno\_fun}  & {\bf 8.811}       & {\bf 0.152}       & 8.833             & {\bf 0.152}       & 8.878             & 0.153 \\
  \textsc{pheno\_go}   & {\bf 4.315}       & {\bf 0.101}       & --                & --                & 4.329             & {\bf 0.101} \\
  \textsc{spo\_fun}    & {\bf 8.806}       & {\bf 0.137}       & 8.845             & {\bf 0.137}       & 8.867             & 0.138 \\
  \textsc{spo\_go}     & {\bf 4.530}       & {\bf 0.083}       & --                & --                & 4.621             & 0.085 \\ \bottomrule
  \end{tabular}
\end{small}
\end{center}
\end{table}
\subsection{Network Flow Prediction}
We consider a vector regression problem where the outputs must satisfy \emph{flow conservation} constraints imposed by a network, $G=(\mathcal{V},\mathcal{A})$. The output space $\mathcal{Y}\subseteq\mathbb{R}^{|\mathcal{A}|}_+$ is characterized by $|\mathcal{A}|$ flow variables $\{y_{ij}\}_{(i,j)\in\mathcal{A}}$, one associated with each arc. Each feasible vector $y\in\mathcal{Y}$ must satisfy the requirement that at every node $j\in\mathcal{V}$, the total inflow is equal to the total outflow,
\begin{equation} \label{eq:flowconserve}
\sum_{k:(j,k)\in\mathcal{A}}y_{jk} - \sum_{k:(k,j)\in\mathcal{A}}y_{kj} = b_j, \forall j\in \mathcal{V}.
\end{equation}

Here $b_j, j\in\mathcal{V}$ are external inflows that are assumed to be known.\footnote{More generally, we can also treat $b_j$ as a decision variable.} These flow constraints can arise naturally from data collected in networked systems with moving entities. For example, $y$ can represent the end-to-end route choices of commuters in a transportation network, the distribution of data packets in a communication network, or the flow of goods in a supply chain. Given inputs $x$, which represent factors that may affect how these entities move in the network, our goal is to predict $y$ while taking into account of known network topology.

We simulate data based on a network shown in Figure \ref{fig:network}, with one source node $s$ and one sink node $t$. The input space $\mathcal{X}$ is a 20-dimensional unit hypercube endowed with a uniform sampling distribution. To simulate the conditional distribution $\mathbb{P}_{Y|X}$, we use a discrete path choice model that assigns to each $s$-$t$ path in the network a random weight that depends on $x$. The flow at each arc is then the sum of the weights of all $s$-$t$ paths that cross it. We compare ECRM under the absolute loss, $\ell_{\text{ab}}(y,y'):=\|y-y'\|_1$ and the square loss, $\ell_{\text{sq}}(y,y'):=\|y-y'\|_2^2$ with two other methods: (i) k-Nearest-Neighbor (kNN): Given an $x$, predict a vector in $\mathcal{Y}$ that minimizes a locally estimated risk, i.e., $\min_{y\in\mathcal{Y}}\sum_{i\in\mathcal{N}(x)}\ell(y,y^{(i)})$, where $\mathcal{N}(x)$ is a set of $k$ nearest samples from $x$. (ii) Kernel Ridge Regression (KRR): First predict a vector $\hat{y}$ by KRR on individual outputs (disregarding the constraints), then project it on $\mathcal{Y}$ by minimizing the Euclidean norm, $\min_{y\in\mathcal{Y}}\|y-\hat{y}\|_2$.

We solve the inference problem for ECRM and kNN under $\ell_{\text{ab}}$ and $\ell_{\text{sq}}$ by a subgradient method and an interior point line-search filter method (using IPOPT, a software library by W{\"a}chter \& Biegler \yrcite{wachter06}), respectively. Both ECRM and KRR use the RBF kernel, and all parameters (including $k$ of kNN) are tuned using a separate validation set.
\begin{figure}[!tbp]
%\vskip 0.2in
\centering
\begin{tikzpicture}
\tikzset{edge/.style = {->,> = latex'}}
\tikzstyle{every node} = [circle, draw=black, fill=white, scale=1.1, inner sep=0pt, minimum size=15pt]
\tikzstyle{dummy} = [draw=none, minimum size=10pt]
\node       (x1)  at (0.0,0.0)   {$s$};
\node       (x2)  at (1.5,1.0)   {};
\node       (x3)  at (1.5,-1.0)   {};
\node       (x4)  at (3.5,1.0)   {};
\node       (x5)  at (3.5,-1.0)    {};
\node       (x6)  at (5.0,0.0)    {$t$};
\node[dummy]       (b1)  at (-1.25,0.0)   {\small $1$};
\node[dummy]       (b6)  at (6.25,0.0)   {\small $1$};
\draw [thick,->,>=stealth] (b1) -- (x1){};
\draw [thick,->,>=stealth] (x1) -- (x2){};
\draw [thick,->,>=stealth] (x1) -- (x3){};
\draw [thick,->,>=stealth] (x2) -- (x3){};
\draw [thick,->,>=stealth] (x2) -- (x4){};
\draw [thick,->,>=stealth] (x3) -- (x5){};
\draw [thick,->,>=stealth] (x2) -- (x5){};
\draw [thick,->,>=stealth] (x3) -- (x4){};
\draw [thick,->,>=stealth] (x4) -- (x5){};
\draw [thick,->,>=stealth] (x4) -- (x6){};
\draw [thick,->,>=stealth] (x5) -- (x6){};
\draw [thick,->,>=stealth] (x6) -- (b6){};
\end{tikzpicture}
\caption{Network used for simulating flows, with source node $s$ and sink node $t$. Here $b_s=1,b_t=-1,$ and $b_j=0,\forall j\not\in\{s,t\}$.}
\label{fig:network}
%\vskip -0.2in
\end{figure}
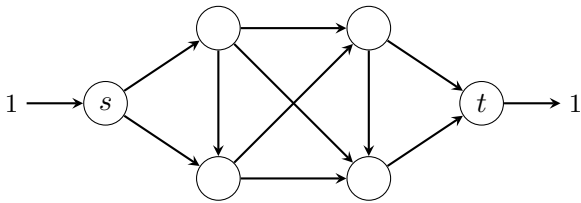

Table \ref{tbl:flow_emprisk} shows the performance of each method when trained and evaluated on samples of 1000 each (averaged over 50 independent trials). ECRM achieves the smallest average risk under both losses, and by a significant margin under $\ell_{\text{ab}}$. The poor performance of KRR under $\ell_{\text{ab}}$ is not surprising given that KRR is based on the square loss. This performance gap with ECRM and KNN underscores the importance of adapting the loss function to the task at hand. Figure \ref{fig:flow_condrisk} shows the conditional risk $R(\hat{h}(x)|x)$ with respect to $\ell_{\text{ab}}$ at a fixed $x$, where $\hat{h}$ is a predictor trained using each method. ECRM closes in on the optimal conditional risk (i.e., $\min_{y\in\mathcal{Y}}R(y|x)$) at $10^{3.5}$ samples, whereas kNN and KRR exhibit significant optimality gaps even with around 10 times more samples. This shows that by minimizing the estimated conditional risk directly, ECRM can produce a prediction that is Bayes optimal.
\section{Conclusions and Future Work}
We have developed a framework for structured prediction based on estimated conditional risk minimization. Our approach treats the problem as one of learning a conditional risk function, which is then minimized by solving an inference problem to predict an output. In particular, we derived a nonparametric family of conditional risk estimators that is based on regularized least squares, and characterized its statistical and computational properties. Unlike existing methods that are based on convex loss relaxations, our approach enables, in some cases, efficient training and inference without having to introduce a surrogate loss. Empirical evaluations on two distinct tasks, one with continuous and another with discrete outputs, demonstrated its effectiveness in adapting to a variety of loss functions.

For future work, we may consider other families of conditional risk estimators and explore a broader class of applications with input-dependent structured outputs. Finally, we note that the applicability of exact linear programming relaxation in our method is not limited to the examples we provided in this paper. For example, our exactness results for HMC can be extended to incorporate additional mutual exclusion constraints (e.g., $y_1=0 \lor y_2=0$), even if total unimodularity is not be satisfied.

\begin{table}[t]
  \caption{The average absolute ($\ell_{\text{ab}}$) and square ($\ell_{\text{sq}}$) losses of each method in the flow prediction task, averaged over 50 independent simulations with 1000 training and test samples in each trial. The standard deviation is shown next to the average.}
  \label{tbl:flow_emprisk}
\begin{center}
\begin{small}
  \begin{tabular}{@{}lcccccc@{}}
  \toprule
  \textsc{loss}    & {\textsc{ECRM}}  & {\textsc{KNN}} & {\textsc{KRR}} \\
  \midrule
  $\ell_{\text{ab}}$      & {\bf 1.6268}\;$\pm$\;0.0347       & 1.7125\;$\pm$\;0.0366       & 1.8673\;$\pm$\;0.0311 \\
  $\ell_{\text{sq}}$      & {\bf 0.6476}\;$\pm$\;0.0197       & 0.6859\;$\pm$\;0.0202       & 0.6495\;$\pm$\;0.0197 \\
  \bottomrule
  \end{tabular}
\end{small}
\end{center}
\end{table}
\begin{figure}
\centering     %%% not \center
\subfigure[Training Time (HMC)]{\label{fig:running_time}\includegraphics[width=40mm]{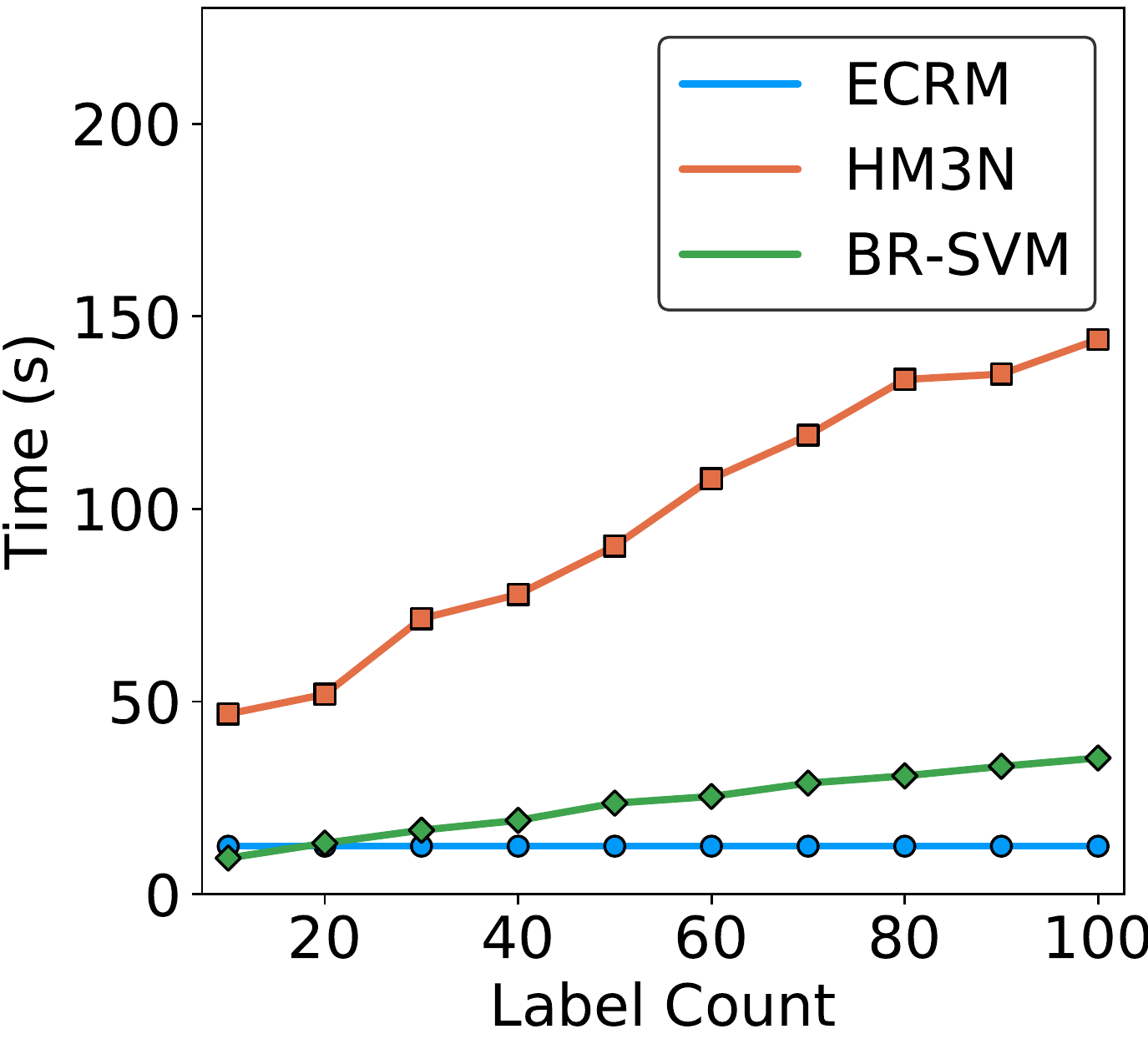}}\hfill
\subfigure[Conditional Risk (Flow)]{\label{fig:flow_condrisk}\includegraphics[width=40mm]{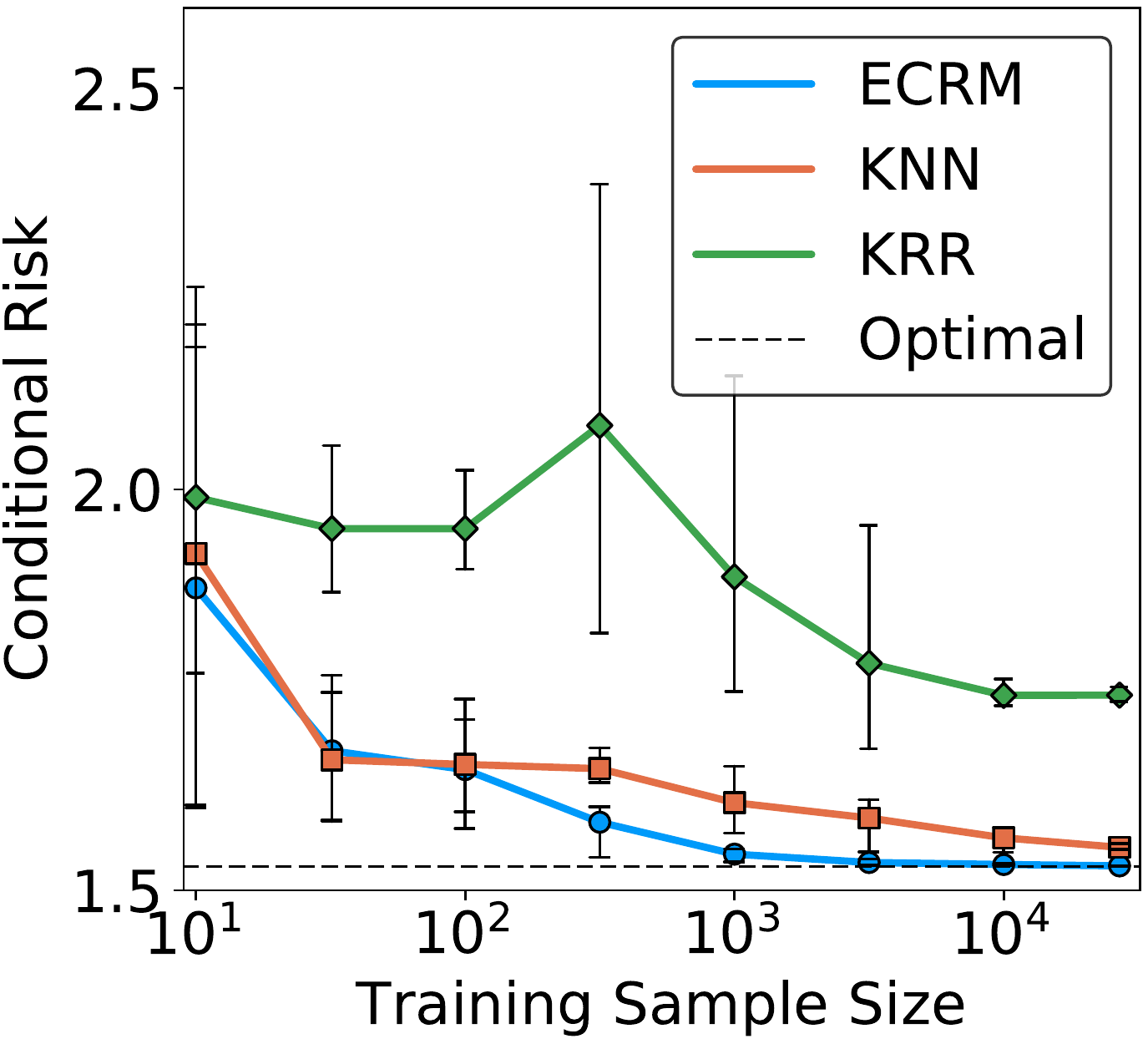}}
\hskip -02.in
\caption{(a) Training time of HMC methods on the {\sc reuters} data set (3000 samples) under hierarchies of varying sizes, obtained by truncating nodes in the original graph. (b) Conditional risks with respect to $\ell_{\text{ab}}$ at $x=0.75$ (element-wise) in the flow prediction task, under various methods and sample sizes. The plots show averages over 20 independent simulations, with the error bars indicating the 15th and 85th quantiles. The dashed line shows the optimal (Bayes) conditional risk.}
\end{figure}

%\section*{Acknowledgements} 
 
%\textbf{Do not} include acknowledgements in the initial version of
%the paper submitted for blind review.

%If a paper is accepted, the final camera-ready version can (and
%probably should) include acknowledgements. In this case, please
%place such acknowledgements in an unnumbered section at the
%end of the paper. Typically, this will include thanks to reviewers
%who gave useful comments, to colleagues who contributed to the ideas, 
%and to funding agencies and corporate sponsors that provided financial 
%support.  


\bibliography{references}
\bibliographystyle{icml2016}
\newpage
\phantom{Placeholder for supplements.}


\section*{Supplementary material}

# Structured Prediction by Conditional Risk Minimization: Supplemental Materials

## 1. Appendix A: Proofs

### 1.1. Proof of Theorem 3.1

**Theorem 3.1.** *If $A$ is totally unimodular and $b \in \mathbb{Z}^n$, then for any $w(x) \in \mathbb{R}^m$, an optimal solution of the inference problem can be found by solving a linear program,*

$$\text{minimize} \quad \sum_{j=1}^{d} \left( \sum_{i=1}^{m} (\ell_j(1, y^{(i)}) - \ell_j(0, y^{(i)})) w_i(x) \right) y_j$$

$$\text{subject to} \quad 0 \leq y_j \leq 1, \forall j = 1, \ldots, d$$

$$y \in \mathcal{Z}$$

*Proof.* First, we reduce the inference problem to an integer linear program. For that purpose, it is useful to express the loss function algebraically in terms of $y$: because $y_j \in \{0, 1\}$, we have $\ell_j(y_j, y') = \ell_j(1, y') y_j + \ell_j(0, y')(1 - y_j), \forall j$. As a shorthand notation, let us denote $\ell_{j,1}^{(i)} := \ell_j(1, y^{(i)})$ and $\ell_{j,0}^{(i)} := \ell_j(0, y^{(i)})$. Then,

$$\min_{y \in \mathcal{Y}} \sum_{i=1}^{m} \left( \sum_{j=1}^{d} \ell_j(y_j, y^{(i)}) \right) w_i(x) = \min_{y \in \mathcal{Y}} \sum_{i=1}^{m} \left( \sum_{j=1}^{d} \ell_{j,1}^{(i)} y_j + \ell_{j,0}^{(i)} (1 - y_j) \right) w_i(x)$$

$$= \min_{y \in \mathcal{Y}} \sum_{j=1}^{d} \left( \sum_{i=1}^{m} (\ell_{j,1}^{(i)} - \ell_{j,0}^{(i)}) y_j w_i(x) + \ell_{j,0}^{(i)} w_i(x) \right)$$

$$= \min_{y \in \mathcal{Y}} \sum_{j=1}^{d} \left( \sum_{i=1}^{m} (\ell_{j,1}^{(i)} - \ell_{j,0}^{(i)}) y_j w_i(x) \right) + \sum_{j=1}^{d} \sum_{i=1}^{m} \ell_{j,0}^{(i)} w_i(x)$$

$$= C + \min_{y \in \mathcal{Y}} \sum_{j=1}^{d} \left( \sum_{i=1}^{m} (\ell_{j,1}^{(i)} - \ell_{j,0}^{(i)}) w_i(x) \right) y_j.$$

Here $C := \sum_j \sum_i \ell_{j,0}^{(i)} w_i(x)$ is a constant that does not depend on the decision variables. Thus the problem is equivalent to minimizing over a linear objective function subject to $y \in \{0, 1\}^d \cap \mathcal{Z}$.

Next, we establish sufficient conditions for the existence of an exact linear programming relaxation of the above integer linear program. This allows us to replace all constraints $y_j \in \{0, 1\}$ with $0 \leq y_j \leq 1, \forall j$, such that the resulting linear program (LP) is guaranteed to contain an optimal solution that is also optimal for the original problem. The key is to show that the set of linear constraints that characterize the LP,

$$a_{i_1}^T y \leq b_{i_1}, \forall i_1 \in I_1, \tag{1}$$

$$a_{i_2}^T y \geq b_{i_2}, \forall i_2 \in I_2, \tag{2}$$

$$a_{i_3}^T y = b_{i_3}, \forall i_3 \in I_3, \tag{3}$$

$$y_j \geq 0, \forall j = 1, \ldots, d, \tag{4}$$

$$y_j \leq 1, \forall j = 1, \ldots, d, \tag{5}$$

can be expressed in a single constraint matrix that is totally unimodular (TU). Let $A' := \begin{bmatrix} A^T|I|I \end{bmatrix}^T$ be the $(n+2d) \times d$ constraint matrix that characterizes the LHS of (1)-(5), where $I \in \mathbb{R}^{d \times d}$ is an identity matrix that corresponds to either (4) or (5). Likewise, let $b' := [b^T|0^T|1^T]^T$ be a $(n+2d) \times 1$ vector that characterizes the RHS, where $0^T$ and $1^T$ are $1 \times d$ row vectors of zeros and ones, respectively. By the assumption of the theorem, we know that $A$ is TU. Our goal now is to establish that $A'$ is also TU.

**Lemma 3.1.** *If $A$ is a TU matrix, then (i) $A^T$ is TU (ii) $[A|I]$ is TU.*

*Proof.* Both of these results are well known. (i) follows directly from the fact that $\det(B^T) = \det(B)$ for any square matrix $B$. A proof of (ii) is available in (Papadimitriou & Steiglitz, 1982). □

By (i) and (ii) in Lemma 3.1, $A \Rightarrow A^T \Rightarrow [A^T|I] \Rightarrow [A^T|I|I] \Rightarrow [A^T|I|I]^T := A'$ is a chain of operations that preserve TU, thus proving the claim that $A'$ is TU.

We now show that this implies every vertex of the polytope characterized by (1)-(5) is integral. At a vertex, there exist $d$ constraints out of (1)-(5) that hold with equality (by default, (3) is included), such that the corresponding constraint vectors are linearly independent. Let $\bar{B}$ and $\bar{b}$ be submatrices of $A'$ and $b'$, respectively, that correspond to these $d$ rows of constraints. Then the vertex is the unique solution of these equalities, $\bar{B}^{-1}\bar{b}$. Note that $\bar{b}$ is integral because $b$ is integral by assumption. Also, by total unimodularity of $A'$, $\det(\bar{B})$ is either $0$, $-1$ or $1$. Therefore, by Cramer's rule, $\bar{B}^{-1}\bar{b}$ is integral.

Together with the well-known fact that any LP defined over a nonempty polytope contains an optimal solution that is a vertex (Bertsimas & Tsitsiklis, 1997), we conclude that the LP defined in Theorem 3.1 has an optimal integer solution that is a vertex. As a result of the relaxation, it must also be optimal for the integer linear program. □

## 1.2. Proof of Proposition 3.2

**Proposition 3.2.** *For any directed graph $G = (\mathcal{V}, \mathcal{A})$, let $A$ be its corresponding hierarchical constraint matrix. Then $A$ is totally unimodular.*

*Proof.* The proof is based on the *equitable bicoloring* property. An equitable bicoloring of a matrix $B$ is a partition of its columns into two (possibly empty) sets $\mathcal{A}$ and $\mathcal{B}$, such that the sum of columns in $\mathcal{A}$ minus the sum of columns in $\mathcal{B}$ is a vector whose entries are $0, 1, -1$. The following theorem establishes the connection between total unimodularity and the existence of equitable bicoloring in a matrix.

**Theorem 1.3.** *(Ghouila-Houri, 1962) A matrix $A$ is totally unimodular if and only if every column submatrix of $A$ admits an equitable bicoloring.*

We show that $A$ satisfies the property of the theorem as follows. By construction, every row of $A$ consists of exactly one entry with $1$ and one entry with $-1$, and zero elsewhere. As a result, given any column submatrix $B$ of $A$, summing up all columns of $B$ results in a vector whose entries are either $0, 1, -1$. Thus $B$ admits an equitable bicoloring (i.e., with partition $\mathcal{A}$ consisting of all columns, and $\mathcal{B} := \emptyset$). Therefore, $A$ is totally unimodular. □

## 1.3. Proof of Proposition 3.3

**Proposition 3.3.** *For any arborescence $G = (\mathcal{V}, \mathcal{A})$ with root $s \in \mathcal{V}$ and any pair $y, y' \in \mathcal{Y}$, the Hierarchical loss $\ell_{hr}(y, y')$ with respect to $G$ is equivalent to*

$$c_s(y_s + y'_s - 2y'_s y_s) + \sum_{(j,k) \in \mathcal{A}} c_k(y'_k y_j + (y'_j - y'_j y'_k - y'_k)y_k).$$

*Proof.* A key observation is that if $G$ is an arborescence, we can simplify $\ell_{\mathrm{hr}}(y, y')$ for any $y, y' \in \mathcal{Y}$ as follows,

$$\ell_{\mathrm{hr}}(y, y') = \sum_{j=1}^{d} c_j \mathbb{1}(y_j \neq y'_j, y_k = y'_k, \forall k \in \mathcal{Q}(j)) \tag{6}$$

$$= c_s \mathbb{1}(y_s \neq y'_s) + \sum_{j \neq s}^{d} c_j \mathbb{1}(y_j \neq y'_j, y_{p(j)} = y'_{p(j)}). \tag{7}$$

Here $p(j)$ is the (unique) parent of node $j$. To see why this is true, we consider each $j$ that is not the root[1] and enumerate all possible cases below. For convenience, we define $\ell_j(y, y') := \mathbb{1}(y_j \neq y'_j, y_k = y'_k, \forall k \in \mathcal{Q}(j))$ and $\bar{\ell}_j(y, y') := \mathbb{1}(y_j \neq y'_j, y_{p(j)} = y'_{p(j)})$.

1. Case 1. If $y_j = y'_j$ or $y_{p(j)} \neq y'_{p(j)}$, then $\ell_j(y, y') = \bar{\ell}_j(y, y') = 0$. Thus we assume in all remaining cases that $y_j \neq y'_j$ and $y_{p(j)} = y'_{p(j)}$.

2. Case 2. If $y_{p(j)} = y'_{p(j)} = 1$, then all remaining ancestors of $j$ must also agree due to the hierarchical constraints, i.e., $y_k = y'_k = 1, \forall k \in \mathcal{Q}(p(j))$. So $y_j \neq y'_j$ implies that $\ell_j(y, y') = \bar{\ell}_j(y, y') = 1$.

3. Case 3. If $y_{p(j)} = y'_{p(j)} = 0$, this contradicts our standing assumption that $y_j \neq y'_j$ because the hierarchical constraint implies that $y_j = y'_j = 0$ must hold. So this case need not be considered.

We have thus proven the equivalence between the two expressions. Now we show how to write it algebraically as a linear function of $y$. For root $s$, we have $\mathbb{1}(y_s \neq y'_s) = y_s(1 - y'_s) + (1 - y_s)y'_s = y_s + y'_s - 2y'_s y_s$. For each remaining node $k$, noting that $(y_{p(k)}, y_k)$ can only take value in $\{(0, 0), (1, 0), (1, 1)\}$, we can express $\mathbb{1}(y_k \neq y'_k, y_{p(k)} = y'_{p(k)})$ as the sum of two terms (in all other cases, the value is zero):

1. $\mathbb{1}(y'_k = 0, y_k = 1, y_{p(k)} = y'_{p(k)} = 1) = (1 - y'_k)y'_{p(k)}y_k$

2. $\mathbb{1}(y'_k = 1, y_k = 0, y_{p(k)} = y'_{p(k)} = 1) = y'_k(y_{p(k)} - y_k)$

Substituting these expressions into (7), we obtain

$$c_s(y_s + y'_s - 2y'_s y_s) + \sum_{k \neq s}^{d} c_k \left( (1 - y'_k)y'_{p(k)}y_k + y'_k(y_{p(k)} - y_k) \right)$$

$$= c_s(y_s + y'_s - 2y'_s y_s) + \sum_{k \neq s}^{d} c_k \left( y'_k y_{p(k)} + (y'_{p(k)} - y'_{p(k)}y'_k + y'_k)y_k \right).$$

Because $G$ is an arborescence, we can enumerate every arc in $\mathcal{A}$ once with $(p(k), k), k \in \mathcal{V}$. Rewriting the summation above over $(j, k) \in \mathcal{A}$ with $p(k) := j$, we complete the proof. $\square$

## 1.4. Proof of Proposition 4.1

**Proposition 4.1.** *The function $L_{\hat{R}}^{\rho}$ satisfies the following properties for any given pair $(x, y)$.*

1. *Surrogacy: $L_{\hat{R}}^{\rho}(x, y) \geq \ell(\hat{h}(x), y)$ for any $\rho > 0$.*

2. *Monotonicity: $L_{\hat{R}}^{\rho}(x, y)$ is nondecreasing in $\rho$.*

3. *Tightness: If $\mathcal{Y}$ is finite, then there exists some $\rho^* > 0$ such that $L_{\hat{R}}^{\rho}(x, y) = \ell(\hat{h}(x), y), \forall \rho \in (0, \rho^*]$.*

---

[1] For the root $s$, $\mathcal{Q}(s) = \emptyset$ and thus $\mathbb{1}(y_s \neq y'_s) = \mathbb{1}(y_s \neq y'_s, y_k = y'_k, \forall k \in \mathcal{Q}(j))$ by definition.

*Proof.* For the first property, observe that $\Phi\left(\max_{y'\in\mathcal{Y}}\left\{\ell(y',y)+\frac{1}{\rho}\Delta_{\hat{R}}(y',x)\right\}\right) \geq \Phi\left(\ell(\hat{h}(x),y)+\frac{1}{\rho}\Delta_{\hat{R}}(\hat{h}(x),x)\right)$. Since $\hat{h}(x)$ is a minimizer of $\hat{R}(\cdot|x)$, we have $\Delta_{\hat{R}}(\hat{h}(x),x) = \min_{y''\in\mathcal{Y}}\hat{R}(y''|x) - \hat{R}(\hat{h}(x)|x) = 0$. Therefore, $L_{\hat{R}}^{\rho}(x,y) \geq \Phi(\ell(\hat{h}(x),y)) = \ell(\hat{h}(x),y)$. For the second property, $\Delta_{\hat{R}}(y',x) \leq 0$ implies that $\ell(y',y)+\frac{1}{\rho}\Delta_{\hat{R}}(y',x)$ is nondecreasing in $\rho$ for any fixed $x,y,y'$. Maximizing the expression over $y'\in\mathcal{Y}$ and applying a monotonic mapping $\Phi$ preserves the nondecreasing property. To prove the third property, let $\mathcal{Y}^*$ denote the set of minimizers of $\hat{R}(\cdot|x)$. If $\mathcal{Y}^* = \mathcal{Y}$, then we are done because for all $\rho\in(0,\infty)$, $\max_{y'\in\mathcal{Y}}\{\ell(y',y)+\frac{1}{\rho}\Delta_{\hat{R}}(y',x)\} = \max_{y'\in\mathcal{Y}}\ell(y',y) = \ell(\hat{h}(x),y)$.[2] Otherwise, we construct $\rho^*$ as follows,

$$\rho^* := \min_{y'\in\mathcal{Y}\backslash\mathcal{Y}^*}\frac{\hat{R}(y'|x)-\min_{y''\in\mathcal{Y}}\hat{R}(y''|x)}{\max\{\ell(y',y)-\ell(\hat{h}(x),y),0\}}.$$

The minimizer $\rho^*$ (possibly $\infty$) above is well-defined and positive because $\mathcal{Y}$ is finite and $\hat{R}(y'|x)-\min_{y''\in\mathcal{Y}}\hat{R}(y''|x) > 0$ for all $y'\in\mathcal{Y}\backslash\mathcal{Y}^*$. It is easy to check that substituting any $\rho\in(0,\rho^*)$ into $\max_{y'\in\mathcal{Y}}\left\{\ell(y',y)+\frac{1}{\rho}\Delta_{\hat{R}}(y',x)\right\}$ guarantees that the optimal value is $\max_{y'\in\mathcal{Y}^*}\ell(y',y) = \ell(\hat{h}(x),y)$. $\square$

## 1.5. Proof of Theorem 4.2

**Theorem 4.2.** *Let $\hat{h}$ be an ECRM predictor trained with some kernel $k$ and regularization parameter $\lambda$. Suppose that $\sup_{y,y'\in\mathcal{Y}}\ell(y,y') \leq L$ and $\sup_{x\in\mathcal{X}}k(x,x) \leq \kappa$. Then for any $\rho > 0, \lambda > 0$ and $\delta\in(0,1)$, the following bound holds with probability at least $1-\delta$,*

$$\mathfrak{R}(\hat{h}) \leq \hat{\mathfrak{R}}^{\rho}(\hat{h}) + \frac{4L\nu}{\rho m} + L\left(\frac{8\nu}{\rho}+1\right)\sqrt{\frac{\ln(1/\delta)}{2m}}, \tag{8}$$

*where $\nu := \kappa/\lambda + (\kappa/\lambda)^{3/2}$.*

The proof is based on establishing *algorithmic stability* (Bousquet & Elisseeff, 2002; Mukherjee et al., 2002) of ECRM, and then applying the generalization bounds for stable learning algorithms. Let $A$ denote an algorithm that takes a training set $S$ and outputs a hypothesis $A(S)\in\mathcal{F}$, where $\mathcal{F}$ is a hypothesis class. For each $i\in\{1,\ldots,m\}$, we define $S^{\backslash i}$ to be the same as $S$, except with the $i$-th sample removed. All training samples are assumed to be i.i.d. The definition of stability that we will use is *uniform stability*.

**Definition** (Uniform Stability) A learning algorithm $A$ has *uniform stability* $\beta$ with respect to a loss function $L : \mathcal{F}\times\mathcal{X}\times\mathcal{Y}\mapsto\mathbb{R}_+$ if $\forall S\in(\mathcal{X}\times\mathcal{Y})^m$ and $\forall i\in\{1,\ldots,m\}$,

$$\sup_{x\in\mathcal{X},y\in\mathcal{Y}}|L(A(S),x,y)-L(A(S^{\backslash i}),x,y)| \leq \beta.$$

Once an algorithm can be proven to have uniform stability, an exponential generalization bound can be derived with the following theorem.

**Theorem 1.5.** *(Bousquet & Elisseeff, 2002) Suppose $A$ is a symmetric learning algorithm[3] that has uniform stability $\beta$ with respect to a loss function $L$ such that $L(A(S),x,y)\in[0,L]$, for all $(x,y)\in\mathcal{X}\times\mathcal{Y}$ and all sets $S$. Then for any $m\geq 1$ and $\delta\in(0,1)$, the following statement holds with probability at least $1-\delta$,*

$$\mathbb{E}_{X,Y}[L(A(S),X,Y)] \leq \frac{1}{m}\sum_{i=1}^{m}L(A(S),X^{(i)},Y^{(i)}) + 2\beta + (4m\beta+L)\sqrt{\frac{1/\delta}{2m}}. \tag{9}$$

---

[2] Recall that we have defined $\hat{h}(x)$ to be a minimizer of $\hat{R}(\cdot|x)$ having the highest loss if multiple minimizers exist, i.e., $\hat{h}(x)\in\max_{y'\in\mathcal{Y}^*}\ell(y',y)$.

[3] An algorithm is said to be symmetric if its output does not depend on ordering of training samples in $S$. ECRM satisfies this property since it is based on KRR, which is symmetric.

In summary, the main step of our proof is to establish that ECRM has uniform stability $\beta = 2L\nu/\rho m$ with respect to $L(\hat{R}_S, x, y) := L^\rho_{\hat{R}_S}(x, y)$, where $\hat{R}_S$ is the conditional risk function learned from training set $S$. Since $L^\rho_{\hat{R}_S} \in [0, L]$ by construction, we can apply Theorem 1.5 by substituting $\beta$ into (9) to obtain a generalization bound in terms of $\mathbb{E}_{X,Y}[L^\rho_{\hat{R}_S}(X, Y)]$. Then, use the fact that $\mathbb{E}_{X,Y}[\ell(\hat{h}(X), Y)] \leq \mathbb{E}_{X,Y}[L^\rho_{\hat{R}_S}(X, Y)]$ to obtain our result.

To simplify the notation, we will denote $\hat{R} := \hat{R}_S$ and $\hat{R}^{\backslash i} := \hat{R}_{S^{\backslash i}}$ from here on. Our proof is based on two main lemmas.

**Lemma 1.6.** *For all $(x, y) \in \mathcal{X} \times \mathcal{Y}$ and for all $i \in \{1, \ldots, m\}$, the following inequality holds,*

$$\left| L^\rho_{\hat{R}}(x, y) - L^\rho_{\hat{R}^{\backslash i}}(x, y) \right| \leq \frac{2}{\rho} \max_{y' \in \mathcal{Y}} \left| \hat{R}(y'|x) - \hat{R}^{\backslash i}(y'|x) \right|.$$

**Lemma 1.7.** *For all $x \in \mathcal{X}$ and for all $i \in \{1, \ldots, m\}$,*

$$\max_{y' \in \mathcal{Y}} \left| \hat{R}(y'|x) - \hat{R}^{\backslash i}(y'|x) \right| \leq \frac{L}{m} \left( \frac{\kappa}{\lambda} + \left( \frac{\kappa}{\lambda} \right)^{3/2} \right).$$

Putting Lemma 1.6 and Lemma 1.7 together immediately implies $\left| L^\rho_{\hat{R}}(x, y) - L^\rho_{\hat{R}^{\backslash i}}(x, y) \right| \leq 2L\nu/\rho m$, which satisfies the definition of uniform stability above. In the next two sections, we will prove these two lemmas.

### 1.5.1. Proof of Lemma 1.6

*Proof.* For any function $F : \mathcal{Y} \times \mathcal{X} \mapsto \mathbb{R}$, denote $\Delta_F(y', x) := \min_{y'' \in \mathcal{Y}} F(y'', x) - F(y', x)$. Observe that

$$
\begin{aligned}
\left| L^\rho_{\hat{R}}(x, y) - L^\rho_{\hat{R}^{\backslash i}}(x, y) \right| &= \left| \Phi \left( \max_{y' \in \mathcal{Y}} \left\{ \ell(y', y) + (1/\rho)\Delta_{\hat{R}}(y', x) \right\} \right) - \Phi \left( \max_{y'' \in \mathcal{Y}} \left\{ \ell(y'', y) + (1/\rho)\Delta_{\hat{R}^{\backslash i}}(y'', x) \right\} \right) \right| \\
&\leq \left| \max_{y' \in \mathcal{Y}} \left\{ \ell(y', y) + (1/\rho)\Delta_{\hat{R}}(y', x) \right\} - \max_{y'' \in \mathcal{Y}} \left\{ \ell(y'', y) + (1/\rho)\Delta_{\hat{R}^{\backslash i}}(y'', x) \right\} \right| \\
&\leq \max_{y' \in \mathcal{Y}} \left| \ell(y', y) + (1/\rho)\Delta_{\hat{R}}(y', x) - \left( \ell(y', y) + (1/\rho)\Delta_{\hat{R}^{\backslash i}}(y', x) \right) \right| \\
&= (1/\rho) \max_{y' \in \mathcal{Y}} \left| \Delta_{\hat{R}}(y', x) - \Delta_{\hat{R}^{\backslash i}}(y', x) \right|,
\end{aligned}
$$

where the first inequality is due to $\Phi$ being a non-expansive mapping, i.e., $|\Phi(a) - \Phi(b)| \leq |a - b|$, and the second equality follows from the fact that for any function $f$ and $g$ (for which a maximizer exists),

$$\min_{y' \in \mathcal{Y}} \{ f(y') - g(y') \} \leq \max_{y' \in \mathcal{Y}} f(y') - \max_{y'' \in \mathcal{Y}} g(y'') \leq \max_{y' \in \mathcal{Y}} \{ f(y') - g(y') \}.$$

Next, we show that $\max_{y' \in \mathcal{Y}} \left| \Delta_{\hat{R}}(y', x) - \Delta_{\hat{R}^{\backslash i}}(y', x) \right| \leq 2 \max_{y' \in \mathcal{Y}} \left| \hat{R}(y'|x) - \hat{R}^{\backslash i}(y'|x) \right|$ as follows.

$$
\begin{aligned}
\max_{y' \in \mathcal{Y}} \left| \Delta_{\hat{R}}(y', x) - \Delta_{\hat{R}^{\backslash i}}(y', x) \right| &= \max_{y' \in \mathcal{Y}} \left| \min_{y'' \in \mathcal{Y}} \hat{R}(y''|x) - \hat{R}(y'|x) - \min_{y''' \in \mathcal{Y}} \hat{R}^{\backslash i}(y'''|x) + \hat{R}^{\backslash i}(y'|x) \right| \\
&\leq \left| \min_{y'' \in \mathcal{Y}} \hat{R}(y''|x) - \min_{y''' \in \mathcal{Y}} \hat{R}^{\backslash i}(y'''|x) \right| + \max_{y' \in \mathcal{Y}} \left| \hat{R}(y'|x) - \hat{R}^{\backslash i}(y'|x) \right|.
\end{aligned}
$$

Since $\left| \min_{y'' \in \mathcal{Y}} \hat{R}(y''|x) - \min_{y''' \in \mathcal{Y}} \hat{R}^{\backslash i}(y'''|x) \right| \leq \max_{y' \in \mathcal{Y}} \left| \hat{R}(y'|x) - \hat{R}^{\backslash i}(y'|x) \right|$, the claim is proven. □

### 1.5.2. Proof of Lemma 1.7

*Proof.* The proof is based on sensitivity analysis of KRR. From Proposition 2.1, for each $y \in \mathcal{Y}$, we can treat $\hat{R}(y|\cdot)$ as the solution of a KRR problem with kernel $k$ and regularization parameter $\lambda$. Our goal is to bound $\|\hat{R}(y|\cdot) - \hat{R}^{\backslash i}(y|\cdot)\|_\infty$, which is essentially a bound on how much a KRR predictor can change due to the removal of one training sample. Let $f$

and $f^{\setminus i}$ be two KRR predictors that are learned from training sets $S$ and $S^{\setminus i}$, respectively,

$$f(\cdot) \in \arg\min_{f' \in \mathcal{H}} \frac{1}{m} \sum_{j=1}^{m} (f'(x^{(j)}) - g(y^{(j)}))^2 + \lambda \|f'\|_{\mathcal{H}}^2 \tag{10}$$

$$f^{\setminus i}(\cdot) \in \arg\min_{f' \in \mathcal{H}} \frac{1}{m-1} \sum_{j \neq i}^{m} (f'(x^{(j)}) - g(y^{(j)}))^2 + \lambda \|f'\|_{\mathcal{H}}^2. \tag{11}$$

Here $g : \mathcal{Y} \mapsto [0, L]$ is a function that maps each $y^{(i)}$ to a bounded nonnegative range. Note that we can simply view (10) and (11) as standard KRR problems with inputs $x$ and bounded real-valued outputs $v := g(y)$. The following proposition establishes several properties of $f$ and $f^{\setminus i}$.

**Proposition 1.8.** *Let $f$ and $f^{\setminus i}$ be two KRR predictors defined in (10) and (11), respectively. If $\sup_{x \in \mathcal{X}} k(x, x) \leq \kappa$, then the following statements hold.*

1. *For all $x \in \mathcal{X}$, $|f(x)| \leq L\sqrt{\kappa/\lambda}$ and $|f^{\setminus i}(x)| \leq L\sqrt{\kappa/\lambda}$.*

2. *$\|f - f^{\setminus i}\|_{\mathcal{H}} \leq \frac{\sqrt{\kappa}\sigma}{2\lambda m}$ holds for any $\sigma$ that satisfies $\forall v \in [0, L]$,*

$$|(f(x) - v)^2 - (f^{\setminus i}(x) - v)^2| \leq \sigma |f(x) - (f^{\setminus i}(x))|.$$

3. *For all $x \in \mathcal{X}$, $|f(x) - f^{\setminus i}(x)| \leq \frac{\kappa\sigma}{2\lambda m}$, where $\sigma$ is defined as in Statement 2.*

*Proof.* To prove the first statement, observe that by the reproducing property of RKHS and Cauchy-Schwarz inequality, $|f(x)| = |\langle f, k(x, \cdot) \rangle_{\mathcal{H}}| \leq \|f\|_{\mathcal{H}} \sqrt{k(x, x)} \leq \|f\|_{\mathcal{H}} \sqrt{\kappa}$. Because $f$ is an optimal solution for (10), we can obtain an upper bound for $\|f\|_{\mathcal{H}}$ with respect to any $f' \in \mathcal{H}$,

$$\lambda \|f\|_{\mathcal{H}}^2 \leq \frac{1}{m} \sum_{j=1}^{m} (f(x^{(j)}) - g(y^{(j)}))^2 + \lambda \|f\|_{\mathcal{H}}^2 \leq \frac{1}{m} \sum_{j=1}^{m} (f'(x^{(j)}) - g(y^{(j)}))^2 + \lambda \|f'\|_{\mathcal{H}}^2$$

Setting $f' = \mathbf{0}$,[4] we obtain $\lambda \|f\|_{\mathcal{H}}^2 \leq \frac{1}{m} \sum_{j=1}^{m} (g(y^{(j)}))^2 \leq L^2$, leading to the bound $\|f\|_{\mathcal{H}} \leq L/\sqrt{\lambda}$ as desired. The same bound can be similarly derived for $f^{\setminus i}$. For a proof of Statement 2, see Lemma 21 and Theorem 22 of Bousquet and Elisseeff (2002). To prove the third statement, we again use the reproducing property and Cauchy-Schwarz inequality, $|f(x) - f^{\setminus i}(x)| = |\langle f - f^{\setminus i}, k(x, \cdot) \rangle_{\mathcal{H}}| \leq \|f - f^{\setminus i}\|_{\mathcal{H}} \sqrt{\kappa}$. Substituting the result from Statement 2 into this bound completes the proof. $\square$

To bound $|f(x) - f^{\setminus i}(x)|$, we still need to define $\sigma$ in Statement 2. Using the fact that $|a^2 - b^2| \leq |a + b||a - b|$, we have

$$|(f(x) - v)^2 - (f^{\setminus i}(x) - v)^2| \leq |f(x) + f^{\setminus i}(x) - 2v||f(x) - f^{\setminus i}(x)|$$
$$\leq (|f(x)| + |f^{\setminus i}(x)| + 2|v|)|f(x) - f^{\setminus i}(x)|.$$

From Statement 1 of Proposition 1.8, we can upper bound $|f(x)|$ and $|f^{\setminus i}(x)|$ with $L\sqrt{\kappa/\lambda}$. Because $|v| \leq L$, we obtain $|f(x)| + |f^{\setminus i}(x)| + 2|v| \leq 2L(\sqrt{\kappa/\lambda} + 1) := \sigma$. Substituting $\sigma$ into Proposition 1.8, we have thus shown that $\forall x \in \mathcal{X}$,

$$|f(x) - f^{\setminus i}(x)| \leq \frac{L}{m} \left( \frac{\kappa}{\lambda} + \left( \frac{\kappa}{\lambda} \right)^{3/2} \right). \tag{12}$$

Now for every $y \in \mathcal{Y}$, define $g(y') := \ell(y, y')$. Note that $\ell$ takes value in $[0, L]$ by assumption. Substituting $g$ into (10)

---

[4] Here $\mathbf{0} \in \mathcal{H}$ denotes the zero vector of the RKHS $\mathcal{H}$.

and (11), we obtain by definition $f(\cdot) := \hat{R}(y|\cdot)$ and $f^{\backslash i}(\cdot) := \hat{R}^{\backslash i}(y|\cdot)$. Therefore, $|\hat{R}(y|x) - \hat{R}^{\backslash i}(y|\cdot)|$ satisfies the upper bound in (12). Because this holds for any $y \in \mathcal{Y}$, we have proven the lemma. □

## 2. Appendix B: Accounting for Intercept in Proposition 2.1

Recall that the optimization problem is the following.

$$\hat{R}(y|\cdot) \in \arg\min_{f \in \mathcal{H}} \frac{1}{m} \sum_{i=1}^{m} (f(x^{(i)}) - \ell(y, y^{(i)}))^2 + \lambda \|f\|_{\mathcal{H}}^2, \tag{13}$$

We briefly discuss how to account for intercept in the above problem. In linear regression, this is often done by centering both the dependent and output variables at their respective empirical means, and then solving the least squares problem without an intercept. The output is then translated by the mean to obtain a final estimate, which can be shown to be equivalent to the output of a model with intercept (Friedman et al., 2001). However, if we apply the same approach to (13), the result may not be what we expect: because the kernel implicitly defines a mapping $\phi(x) \mapsto \mathcal{H}$ from the original input space to a feature space, which can be highly nonlinear, a point that is centered in $\mathcal{X}$ need not be centered in $\mathcal{H}$. This can result in biases.

One alternative is to center the inputs in the feature space by working only with the inner products (Meilă, 2002). Let us define $\tilde{\phi}_i := \phi_i - \bar{\phi}$ to be a centered input in the feature space, where $\phi_i := \phi(x^{(i)})$ and $\bar{\phi} := \frac{1}{m} \sum_{i=1}^{m} \phi(x^{(i)})$. Then, the gram matrix can be characterized by

$$\begin{aligned}
\langle \tilde{\phi}_i, \tilde{\phi}_j \rangle_{\mathcal{H}} &= \langle \phi_i - \bar{\phi}, \phi_j - \bar{\phi} \rangle_{\mathcal{H}} \\
&= \langle \phi_i, \phi_j \rangle_{\mathcal{H}} - \langle \phi_i, \bar{\phi} \rangle_{\mathcal{H}} - \langle \phi_j, \bar{\phi} \rangle_{\mathcal{H}} + \langle \bar{\phi}, \bar{\phi} \rangle_{\mathcal{H}} \\
&:= \tilde{k}(x^{(i)}, x^{(j)}).
\end{aligned}$$

It can be shown that the resulting $\tilde{k}$ is also a positive definite kernel, and its gram matrix $\tilde{K}$ is a *centered kernel matrix* (Cortes et al., 2012) that can be written as

$$\tilde{K} = \left( I - \frac{11^T}{m} \right) K \left( I - \frac{11^T}{m} \right),$$

where $1^T$ is a row vector of ones of appropriate dimension. The output variables can be centered as follows. For each $y$, let us define $\bar{\ell}_y := \frac{1}{m} \sum_{i=1}^{m} \ell(y, y^{(i)})$ to be the empirical mean of losses. We center the vector $L_y$ defined in the proof of Proposition 2.1 by letting $\tilde{L}_y := L_y - \bar{\ell}_y 1$. Then, just as in linear regression, we estimate $\hat{R}(y|x)$ by solving for $\alpha_y^*$ with all the centered inputs, and then translate the result by $\bar{\ell}_y$,

$$\begin{aligned}
\hat{R}(y|x) &= \tilde{L}_y^T (\tilde{K} + m\lambda I)^{-1} \tilde{v}(x) + \bar{\ell}_y \\
&= \sum_{i=1}^{m} (\tilde{w}_i(x) - u(x)) \ell(y, y^{(i)}) + \frac{1}{m} \sum_{i=1}^{m} \ell(y, y^{(i)}).
\end{aligned}$$

Here $\tilde{w}(x) := (\tilde{K} + m\lambda I)^{-1} \tilde{v}(x)$, $\tilde{v}(x) := [\tilde{k}(x, x^{(i)})]_{i=1}^{m}$ and $u(x) := \frac{1}{m} \sum_{i=1}^{m} \tilde{w}_i(x)$. In the above expression, we can interpret the second term $\frac{1}{m} \sum_{i=1}^{m} \ell(y, y^{(i)})$ as a baseline sample average approximation of the risk function, while the first term as the correction after observing $x$.

# 3. Appendix C: Summary of Data Sets

*Table 1.* A summary of the attributes of each data set. Cardinality is the average number of labels per sample. Max Depth and Avg Depth correspond to the depth of the hierarchy and the average depth of the labels, respectively.

| Data set | #Features | #Labels | #Train | #Test | Cardinality | Max Depth | Avg Depth |
|---|---|---|---|---|---|---|---|
| ENRON | 1001 | 56 | 988 | 660 | 5.30 | 2 | 1.18 |
| REUTERS | 47236 | 103 | 3000 | 3000 | 3.23 | 3 | 1.40 |
| WIPO | 74435 | 188 | 1352 | 358 | 4.00 | 3 | 2.80 |
| IMCLEF07A | 80 | 96 | 2000 | 1006 | 3.00 | 2 | 1.57 |
| IMCLEF07D | 80 | 46 | 2000 | 1006 | 3.00 | 2 | 1.48 |
| PHENO_FUN | 276 | 300 | 1009 | 581 | 8.86 | 5 | 2.16 |
| PHENO_GO | 276 | 296 | 1005 | 581 | 5.44 | 10 | 4.01 |
| SPO_FUN | 84 | 383 | 2437 | 1266 | 8.71 | 5 | 2.25 |
| SPO_GO | 84 | 508 | 2434 | 1263 | 5.58 | 10 | 4.24 |

The data sets are available at these sources:

1. http://kt.ijs.si/DragiKocev/PhD/resources/doku.php?id=hmc_classification

2. https://dtai.cs.kuleuven.be/clus/hmcdatasets

3. https://www.csie.ntu.edu.tw/~cjlin/libsvmtools/datasets/multilabel.html
   (REUTERS topic hierarchy, subset1)[5]

For IMCLEF07A and IMCLEF07D, we only used the first 2000 training samples out of the 10000. In the last 4 data sets, we combined the training and validation set into a single training set for cross-validated parameter tuning. We also trimmed down the hierarchy for these 4 data sets by discarding labels with less than 3 positive instances. For PHENO_GO and SPO_GO, we use only the first connected component of the full hierarchy.

---

[5]The topic hierarchy is available at http://www.jmlr.org/papers/volume5/lewis04a/lyrl2004_rcv1v2_README.htm.


\end{document}